\documentclass[letterpaper]{article} 
\usepackage{aaai24}  
\usepackage{times}  
\usepackage{helvet}  
\usepackage{courier}  
\usepackage[hyphens]{url}  
\usepackage{graphicx} 
\usepackage{natbib}  
\usepackage{caption} 
\usepackage{algorithm}
\usepackage{algorithmic}

\usepackage{newfloat}
\usepackage{listings}
\DeclareCaptionStyle{ruled}{labelfont=normalfont,labelsep=colon,strut=off} 
\floatstyle{ruled}
\newfloat{listing}{tb}{lst}{}
\floatname{listing}{Listing}
\usepackage{color}
\usepackage{amsthm,amsmath,amssymb}
\usepackage{mathrsfs}
\usepackage{chngcntr}
\usepackage{bm}
\usepackage{subfigure}
\usepackage{multirow}
\usepackage{amsfonts}
\usepackage{booktabs}
\usepackage{bbding}
\usepackage{graphicx}
\usepackage[table]{xcolor}
\newtheorem{definition}{Definition}
\newtheorem{theorem}{Theorem}

\title{FedA$^3$I: Annotation Quality-Aware Aggregation for Federated Medical Image Segmentation against Heterogeneous Annotation Noise}
\author{
    Nannan Wu,
    Zhaobin Sun,
    Zengqiang Yan\thanks{Corresponding author},
    Li Yu
}
\affiliations{
    School of Electronic Information and Communications, Huazhong University of Science and Technology, Wuhan, China\\

    wnn2000@hust.edu.cn, zbsun@hust.edu.cn, z\_yan@hust.edu.cn, hustlyu@hust.edu.cn
}

\usepackage{bibentry}

\begin{document}

\maketitle

\begin{abstract}
	
	Federated learning (FL) has emerged as a promising paradigm for training segmentation models on decentralized medical data, owing to its privacy-preserving property. However, existing research overlooks the prevalent annotation noise encountered in real-world medical datasets, which limits the performance ceilings of FL. In this paper, we, for the first time, identify and tackle this problem. For problem formulation, we propose a contour evolution for modeling non-independent and identically distributed (Non-IID) noise across pixels within each client and then extend it to the case of multi-source data to form a heterogeneous noise model (\textit{i.e.}, Non-IID annotation noise across clients). 
	For robust learning from annotations with such two-level Non-IID noise, we emphasize the importance of data quality in model aggregation, allowing high-quality clients to have a greater impact on FL. 
	To achieve this, we propose \textbf{Fed}erated learning with \textbf{A}nnotation qu\textbf{A}lity-aware \textbf{A}ggregat\textbf{I}on, named \textbf{FedA$^3$I}, by introducing a quality factor based on client-wise noise estimation. Specifically, noise estimation at each client is accomplished through the Gaussian mixture model and then incorporated into model aggregation in a layer-wise manner to up-weight high-quality clients. Extensive experiments on two real-world medical image segmentation datasets demonstrate the superior performance of FedA$^3$I against the state-of-the-art approaches in dealing with cross-client annotation noise.
	The code is available at https://github.com/wnn2000/FedAAAI.
	
\end{abstract}

\section{Introduction}

Accurate medical image segmentation is crucial as it is typically the first step in analytical workflows. Recent developments in segmentation suggest collecting massive data to train deep learning (DL) models \cite{ronneberger2015u,gu2019net,valanarasu2022unext}. However, growing privacy concerns in medical scenarios prohibit data transfer among sources \cite{rieke2020future,dou2021federated,wu2022federated,yan2020variation,li2021fedbn}, making it infeasible to construct datasets at scale. Fortunately, federated medical image segmentation (FMIS) \cite{liu2021feddg} based on federated learning (FL) \cite{mcmahan2017communication,huang2023generalizable,huang2023federated} offers a solution by enabling multiple medical institutions to train a unified segmentation model cooperatively without data sharing. As FMIS enriches training data in a privacy-preserving manner, it has emerged as a promising paradigm in practice.

\begin{figure}[!t] 
	\centering
	\includegraphics[width=\columnwidth]{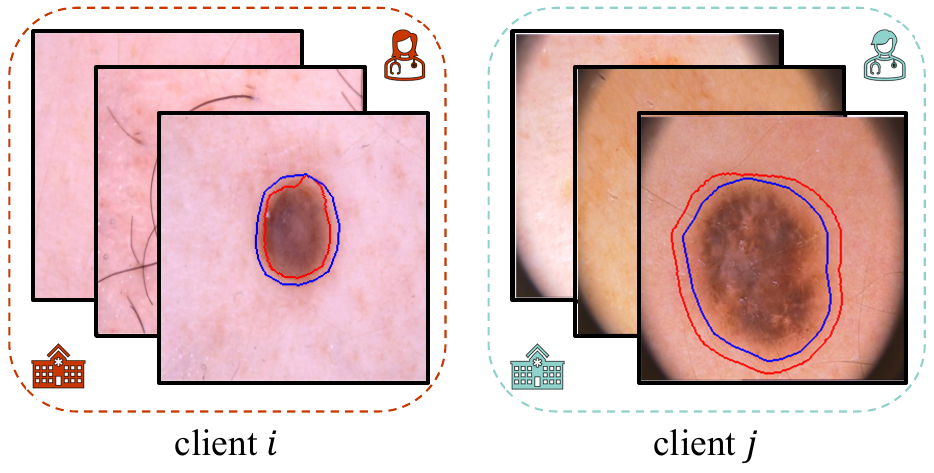}
	\caption{Annotation noise in multi-source datasets. The blue and red curves represent the contours of clean (\textit{i.e.}, ideal) and noisy annotations, respectively. Given a certain sample, annotations between two curves are noisy, indicating that noise is distributed near contours instead of being IID among pixels. Given samples from different clients, annotation noise is heterogeneous where noise on client $i$ causes under-segmentation and noise on client $j$ causes over-segmentation.}
	\label{background}
\end{figure}
Existing research on FMIS has mainly focused on distribution shift \cite{liu2021feddg,xu2022closing,wang2022personalizing,jiang2023iop,wang2023feddp}, client-side unfairness \cite{jiang2023fair}, and region-of-interest (ROI) inconsistency \cite{xu2023federated} caused by decentralized data. However, prevalent noisy annotations, referring to annotations containing errors, in medical datasets \cite{ji2021learning,zhang2020disentangling} have often been overlooked, which are inevitable in real-world medical datasets. On the one hand, annotating is subjective that can be biased by annotators' expertness (\textit{i.e.}, \textbf{preference}). For instance, front-line clinicians tend to mark ROIs larger to avoid missing positive parts (\textit{i.e.} tumors/lesions) in ambiguous regions \cite{calisto2021introduction,salehi2017tversky}, while non-professionals may only recognize well-characterized core areas, leading to smaller annotated regions. On the other hand, the level of annotator engagement can also affect the quality of annotations, resulting in random mislabeling (\textit{i.e.}, \textbf{randomness}) when annotators are distracted or fully occupied. It should be noted that noisy annotations can seriously degrade the performance of segmentation models under supervised learning paradigms \cite{liu2022adaptive}. Thus, addressing cross-client annotation noise in FMIS is an urgent task demanding insightful investigation.

Despite extensive studies on addressing related issues like
\begin{enumerate}
	\item Noisy labels for classification \cite{xu2022fedcorr,wu2023fednoro} and weak annotations for segmentation  \cite{wicaksana2022fedmix,zhu2023feddm} in FL,
	
	\item Noisy annotations for segmentation in centralized learning (CL) \cite{liu2022adaptive,yao2023learning},
\end{enumerate}
noisy annotations in FMIS may not be well addressed based on existing methods. 
It is because those methods fail to address the two-level Non-IID noise in annotations across both pixels and clients in FMIS, as shown in Fig. \ref{background}. Specifically, at the pixel level, annotation noise tends to be distributed in indistinct regions near contours, caused by inaccurate contour delineation, making it difficult to analyze the noise as the first type of studies mainly focuses on addressing relatively uniform noise within each classification instance. As for the client level, as annotation noise is highly annotator-dependent, noise is heterogeneous across clients due to non-overlapping annotators under the context of multi-source data in FL \cite{wu2023fednoro}. Such cross-client scenario makes it almost impossible to model annotation noise from a universal perspective, posing additional challenges when applying strategies from the second type of studies that primarily address homogeneous annotation noise. Consequently, how to effectively address noisy annotations in FMIS remains unexplored.

In this paper, we, for the first time, identify and formulate this new FMIS problem, \textit{i.e.}, how to train a federated segmentation model with noisy pixel-wise annotations. Different from weak annotations like bounding boxes studied in previous FMIS works \cite{wicaksana2022fedmix,zhu2023feddm}, noisy annotations studied in this paper are more common and realistic in clinical practice and thus more worthy of being addressed. 
To formulate annotation noise, inspired by the real-world annotating process of delineating ROI contours, we first propose a generic model to generate noisy annotations named contour evolution model (CEM). In CEM, noise is generated by making a clean contour evolve into a corrupted one, determined by two main parameters corresponding to annotators' preference and randomness. Following a CEM with certain parameters, annotation noise within a client/annotator is produced where noise among pixels is Non-IID. To generate heterogeneous annotation noise in multi-source data, different clients are assigned with CEMs with different parameters, resulting in Non-IID annotation noise across clients.
To address such two-level Non-IID annotation noise in FMIS, our insight resides in that model aggregation in FL should consider not only data quantity as FedAvg \cite{mcmahan2017communication} but also data quality (\textit{i.e.}, annotation quality in this paper). Therefore, we incorporate \textbf{Fed}erated learning with \textbf{A}nnotation qu\textbf{A}lity-aware \textbf{A}ggregat\textbf{I}on, named \textbf{FedA$^3$I}, for noise-robust FMIS. In FedA$^3$I, the annotation noise of each client is estimated by the Gaussian mixture model (GMM) with the learning difficulty of regions around contours, which is used to compute the quality-based aggregation weight. This step is conducted on the region-level rather than the pixel-level to mitigate the impact of Non-IID noise across pixels on noise analysis. After obtaining clients' quality- and quantity-based (\textit{i.e.}, FedAvg) weights, instead of directly adjusting the corresponding client-wise weights for model-wise aggregation \cite{liang2022rscfed,wicaksana2022fedmix,wu2023fednoro}, we dynamically combine the two weights in a layer-wise manner \cite{ma2022layer}. To be specific, quantity-based weights contribute more to updating shallower layers while quality-based weights dominate the updates of deeper layers. It is motivated by the consensus that noisy labels have a more severe impact on deep layers \cite{bai2021understanding}.
The main contributions are as follows:

\begin{itemize}
	
	\item A new FMIS problem where two-level Non-IID noise is identified to be addressed for the first time.
	
	\item A generic annotation noise model, named CEM, to formulate an annotator's preference and randomness, and a heterogeneous annotation noise model to formulate cross-annotator/-client annotation noise based on CEMs with different parameters.
	
	\item A novel FL framework FedA$^3$I to address two-level Non-IID noise in FMIS, where both data quantity and quality are included for layer-wise model aggregation.
	
	\item Superior performance against the state-of-the-art (SOTA) approaches validated on two publicly-available real-world medical datasets.
	
\end{itemize}

\begin{figure*}[!t] 
	\centering
	\includegraphics[width=1\textwidth]{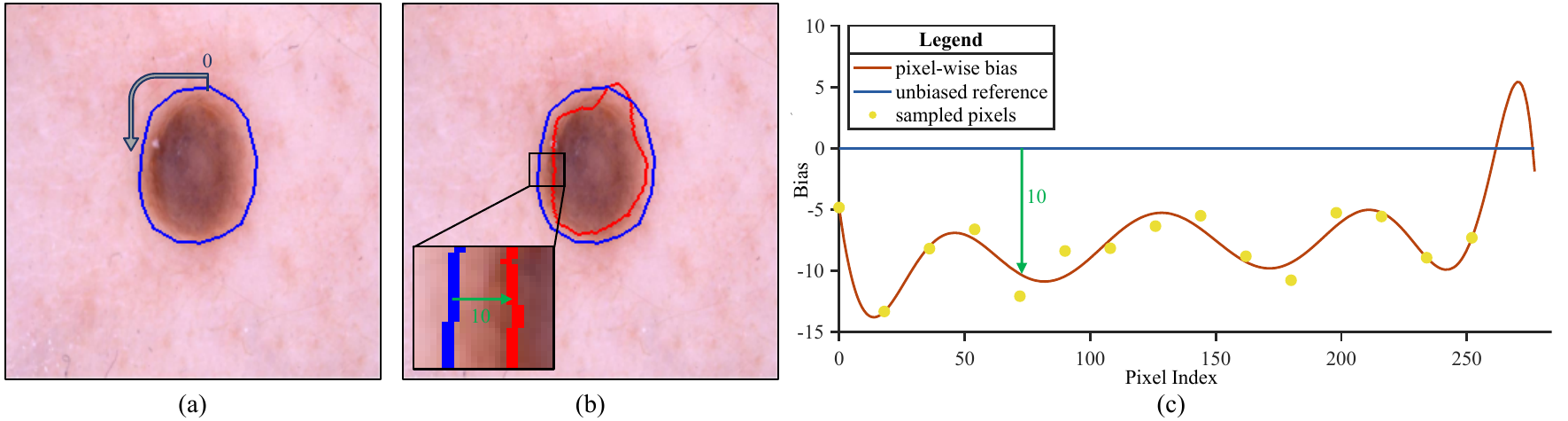}
	\caption{An example of using CEM (\textit{i.e.}, $C(-8, 2)$) to generate annotation noise. In (a), the arrow indicates the positive annotation direction starting from the origin (\textit{i.e.}, 0) of the contour (\textit{i.e.}, the blue curve). Then, the bias of each pixel is calculated based on the polynomial function fitted by sampled pixels as illustrated in (c), and used to control the movements of pixels of the contour as indicated by the green arrow, resulting in the noisy contour (\textit{i.e.}, the red curve) in (b).}
	\label{cem}
\end{figure*}

\section{Related Work}

\subsection{Federated Medical Image Segmentation}

FMIS has shown great potential for decentralized medical datasets under privacy preservation. Existing research on FMIS mainly emphasizes addressing data heterogeneity and its derived problems, \textit{e.g.}, distribution shift \cite{liu2021feddg,xu2022closing,wang2022personalizing,jiang2023iop,wang2023feddp}, inconsistent ROIs \cite{xu2023federated} and inter-client unfairness \cite{jiang2023fair} Recently, some studies \cite{wicaksana2022fedmix,zhu2023feddm} address weak annotations in FMIS. FedMix \cite{wicaksana2022fedmix} explores how to learn from mixed-type annotations including complete and weak annotations, while FedDM \cite{zhu2023feddm} deals with typical bounding box annotations. 

\subsection{Learning with Noisy Labels/Annotations}

It is important to design noise-robust training strategies given the inevitable mislabeling phenomenon. Existing studies on this are conducted under both CL \cite{han2018co,yu2019does,wang2019symmetric,zhang2018generalized,wang2020noise,liu2020early,liu2022adaptive,yao2023learning,fang2023reliable,lidividemix} and FL paradigms \cite{xu2022fedcorr,wu2023fednoro,jiang2022towards}, mainly based on sample selection \cite{han2018co,yu2019does}, noise-robust loss functions \cite{wang2019symmetric,zhang2018generalized,wang2020noise}, noise-robust regularization \cite{liu2020early,jiang2022towards}, label correction \cite{liu2022adaptive,xu2022fedcorr,yao2023learning}, label softening \cite{wu2023fednoro,fang2023reliable}, semi-supervised learning \cite{lidividemix}, \textit{etc}.

\section{Methodology}

\subsection{Preliminaries}

For simplification, this paper focuses on the typical binary segmentation problem\footnote{This can be easily generalized to multi-class segmentation by studying each class separately.} following previous studies \cite{wicaksana2022fedmix,zhu2023feddm}. Given $K$ participants in FL, each client $k$ holds a private dataset $D_k=\{(x_i \in \mathcal{X} \subseteq \mathbb{R}^{\rm H \times W \times C}, \tilde{y}_i \in \mathcal{Y}=\{0,1\}^{\rm H \times W})\}_{i=1}^{n_k}$, where $n_k$ represents the size of $D_k$, and $(x_i, \tilde{y}_i)$ denotes an available image-annotation pair with the height, width, and channel dimensions of $\rm H$, $\rm W$, and $\rm C$, respectively. Note that the annotation $\tilde{y}_i$ is noisy due to the inevitable mislabeling phenomenon, described as $\tilde{y}_i = y_i \oplus \xi_i$, where $y_i$ is the underlying true annotation, $\xi_i \in \mathcal{Y}$ represents the pixel-wise noise, and $\oplus$ is the pixel-wise binary XOR operation. In our work, we focus on two-level Non-IID noise, \textit{i.e.}, the pixel level and the client level. Our goal is to develop a noise-robust FL algorithm to train a binary segmentation model $f_\theta: \mathcal{X} \rightarrow \mathbb{R}^{\rm H \times W}$ parameterized by $\theta$.

\subsection{Noise Model} \label{sec:NoiseModel}

To formulate the problem, we first present a noise model to simulate pixel-wise annotations and then generalize it to the case of multi-source datasets.

\subsubsection{Contour Evolution Model}

Inspired by the clinical practice where annotating is typically performed by delineating ROI contours and noise is essentially introduced by biased delineation, we propose a contour evolution model (CEM) to describe annotation noise as a biased contour. 
In this way, contour bias can be quantified as a pixel bias sequence $\langle b_k \rangle_{k=1}^l = \langle b_1, b_2, \dots, b_l\rangle$, where $l$ is the number of contour pixels. Each pixel moves outward (with a positive bias) or inward (with a negative bias) along the normal direction according to the absolute value of its bias, forming a corrupted contour. To further analyze the bias, we decompose it into two orthogonal components as $b_k = c+r_k$, where $c$ is a fixed component and $r_k$ is the $k$-th element of a random component $\langle r_k \rangle_{k=1}^l$ with $\sum_{k=1}^l r_k = 0$. These two components reflect two real-world factors contributing to the bias, namely preference and randomness. In clinical practice, annotators may draw contours of different sizes based on expertise and preference \cite{zhang2020disentangling}, resulting in annotated objects being larger (\textit{i.e.} $c>0$) or smaller (\textit{i.e.}, $c<0$) than ground truth. As for randomness, annotators may make mistakes, resulting in random distortion being introduced along contours (\textit{i.e.}, $r_k$).

Based on the analysis above, we can generate annotator-specific annotation noise by controlling the two components. Specifically, given $l$ contour pixels, a subset of $l_{sub}$ pixels are selected with equal intervals, and $l_{sub}$ values are sampled from a normal distribution $\bm{N}(\mu,\sigma^2)$. Then, $l_{sub}$ values, along with their $l_{sub}$ indices, are employed to fit a $p$-th degree polynomial function $\mathscr{P}$ that is used to recompute the bias sequence with indices, \textit{i.e.}, $\langle b_k \rangle_{k=1}^l = \langle \mathscr{P}(1), \mathscr{P}(2), \dots, \mathscr{P}(l) \rangle$. With this design, a pixel-wise bias with a fixed component approximately equal to $\mu$ and a random component controlled by $\sigma$ is generated. In addition, the bias sequence $\langle b_k \rangle_{k=1}^l$ is continuous, ensuring contour continuity and smoothness. Each contour pixel is then moved along the normal direction to a new position based on the corresponding element in $\langle b_k \rangle_{k=1}^l$, resulting in annotator-specific noise controlled by $(\mu, \sigma)$. For convenience, we denote such a contour evolution model as $C(\mu, \sigma)$ as illustrated in Fig. \ref{cem}.

Compared to existing noise models like class-conditional noise (CCN) \cite{han2018co,yu2019does,chen2021beyond}, morphological operation noise (MON) \cite{zhang2020disentangling,liu2022adaptive}, and Markov label noise (MLN) \cite{yao2023learning}, CEM generates more realistic annotation noise as it directly simulates the real-world annotating process. Additionally, CEM considers the Non-IID nature of noise at the pixel level more comprehensively. To illustrate this, we revisit different noise models from the pixel-dependent noise (PDN) perspective by introducing the signed distance function \cite{yao2023learning}.
\begin{definition}[PDN] \label{PDN}
	\normalfont For an underlying clean annotation, denote the set of all pixels as $\mathcal{A}$ and the signed distance function as $\mathscr{D}: \mathcal{A} \rightarrow \mathbb{R}$. Assume the imposed noise $\xi$ is drawn from a random matrix $\mathbf{\Xi}=(\Xi_{i,j})$.
	$\xi$ is pixel-dependent if
	\begin{enumerate}
		\item $\exists \epsilon>0$, s.t. $\sum_{(i,j) \in \mathcal{I}} \frac{\mathbb{E}[\Xi_{i,j}] }{|\mathcal{I}|}>\sum_{(i,j) \in \mathcal{O}} \frac{\mathbb{E}[\Xi_{i,j}]}{|\mathcal{O}|}$, where $\mathcal{I}=\{(i,j) \in \mathcal{A} \mid |\mathscr{D}((i,j))| < \epsilon\}$, and $\mathcal{O}=\mathcal{A} \setminus \mathcal{I}$.
		
		\item $\exists \delta \in \mathbb{R}$ and two different pixels $(u_1,v_1),(u_2,v_2) \in \mathcal{A}$, satisfying $\mathscr{D}((u_1,v_1))=\mathscr{D}((u_2,v_2))=\delta$, s.t. $\Xi_{u_1,v_1}$ and $\Xi_{u_2,v_2}$ are non-identically distributed.
	\end{enumerate}
\end{definition}
The first condition ensures that noise is localized near contours, and the second condition ensures that noise is heterogeneous along contours. PDN indicates that noise in pixel-wise annotations is Non-IID across both the normal and tangential directions of contours, which fits real-world scenarios due to the following two reasons. First, when an annotator segments an object by marking its contours, noise is introduced to the regions between the marked and the underlying clean contours. That is, from the normal direction, noise mainly exists around contours. Second, the difficulty of marking different contour parts could vary since medical objects often lack isotropy. As a consequence, Non-IID noise along the tangential direction could arise. However, existing noise models (\textit{e.g.}, CCN, MON, MLN) fail to guarantee pixel-dependent due to the strong assumption of noise distribution.
Comparatively, CEM well satisfies the above conditions by formulating annotation noise as a biased contour and altering the bias of different contour parts flexibly. More discussions and proof can be found in Appendix\footnote{The appendix is included in the arXiv version, which can be found in https://github.com/wnn2000/FedAAAI} A.

\subsubsection{Heterogeneous Noise Model with CEMs}

In FMIS, annotators are non-overlapping across clients, resulting in heterogeneous noise due to its annotator-dependent nature. Based on CEM to generate annotator-specific noise, we present a heterogeneous noise model to emulate real-world noise in multi-source datasets by assigning CEMs with different parameters to clients. For each client $i$, we use a specific CEM $C(\mu_i, \sigma_i)$ to generate its annotation noise by setting
\begin{equation} 
	\mu_i =
	\left\{
	\begin{aligned}
		\mu \sim \bm{U}(0, \mu^{max}), \quad&\text{ with probability } p_d  \\
		\mu \sim \bm{U}(\mu^{min}, 0), \quad&\text{ with probability } 1-p_d,
	\end{aligned}
	\right.
\end{equation}
and $\sigma_i = \sigma \sim \bm{U}(\frac{\sigma^{max}}{2}$, $\sigma^{max})$, where $\bm{U}$ denotes the uniform distribution. In this way, different clients can generate heterogeneous noise according to client-dependent CEMs. Such a multi-source noise model is denoted as $M(\mu^{max}, \mu^{min}, \sigma^{max}, p_d)$ for simplification.

\subsection{Annotation Quality-Aware Aggregation}

\subsubsection{Motivation and Overview}

When adopting FedAvg \cite{mcmahan2017communication} for FMIS, the optimization objective can be formulated as
\begin{equation} \label{FedAvg}
	\min_\theta [F(\theta) := \frac{\sum_{i=1}^{K} n_i F_i(\theta, D_i)}{\sum_{i=1}^{K} n_i}],
\end{equation}
where $F_i(\theta, D_i)$ represents the local objective function of client $i$. The contribution of each client to the global model is determined solely by its data amount $n_i$ in FedAvg. Consequently, the global model can be negatively diverted by clients with large amounts of noisy annotations. To address this, we propose FedA$^3$I to make global model aggregation quality-dependent, thereby making high-quality clients dominate the training process as illustrated in Fig. \ref{Framework}. Details are presented as follows.

\begin{figure}[!t] 
	\centering
	\includegraphics[width=\columnwidth]{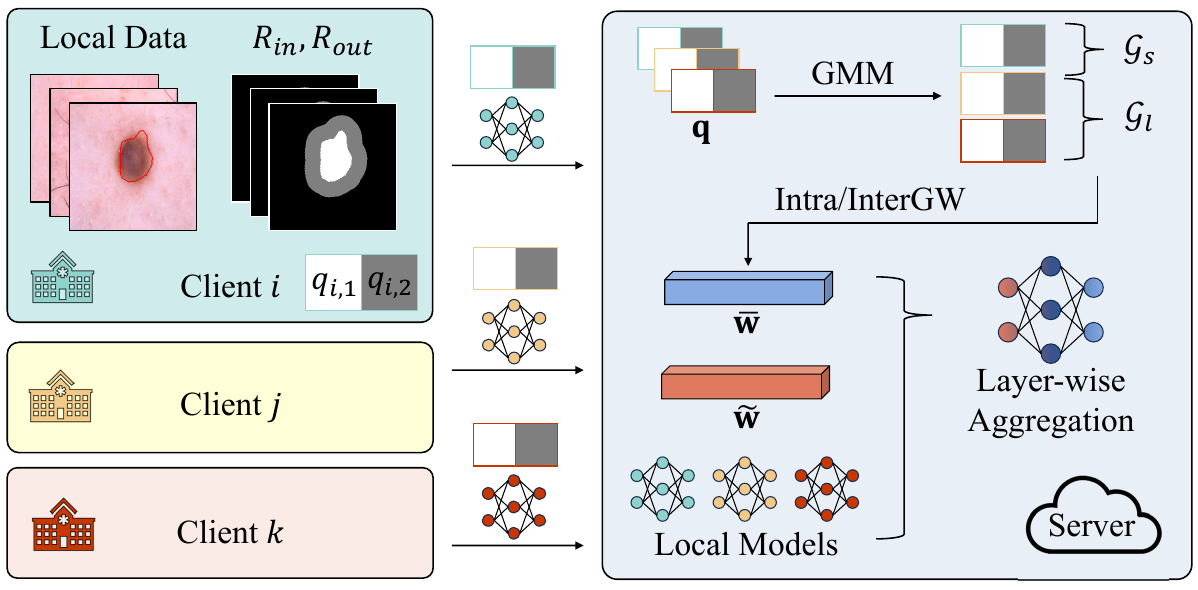}
	\caption{Overview of FedA$^3$I. In each client $i$, we compute the learning difficulty of the inner and outer regions (the white and gray regions around contours respectively), denoted as $q_{i,1}$ and $q_{i,2}$, and upload them to the server. These indicators are used to fit a Gaussian mixture model (GMM), which divides all clients into two subsets. Based on this, we compute the quality-based weights using two components, namely IntraGW and InterGW. Finally, both quality-based and quantity-based weights are utilized for model aggregation in a layer-wise manner.}
	\label{Framework}
\end{figure}

\subsubsection{Noise Estimation}

We analyze annotation quality by examining the noise in each client. As annotation noise is heterogeneous across clients and annotated objects can be either larger or smaller than ground truth, the global model typically learns an intermediate pattern that is relatively clean for a faster drop in loss minimization during the early stage of training. This phenomenon, known as early learning \cite{liu2022adaptive,lidividemix,liu2020early}, allows us to obtain a model that is relatively free from noise and has preliminary segmentation ability for noise estimation. Specifically, assuming an early model $f_{\theta_e}$ warmed up by FedAvg with $T_1$ rounds, we first calculate the client-wise learning difficulty $\mathbf{q} = (q_{i,j}) \in \mathbb{R}^{K \times 2}$ with

\begin{equation} \label{difficulty1}
	q_{i,1} = \frac{1}{n_i} \sum_{(x,\tilde{y}) \in D_i} \sum_{(u,v) \in {R}_{in}(\tilde{y})}  \frac{\ell_{ce}(f_{\theta_e}(x), \tilde{y})_{u,v}}{{|{R}_{in}(\tilde{y})|}}
\end{equation}
and
\begin{equation} \label{difficulty2}
	q_{i,2} = \frac{1}{n_i} \sum_{(x,\tilde{y}) \in D_i} \sum_{(u,v) \in {R}_{out}(\tilde{y})} \frac{\ell_{ce}(f_{\theta_e}(x), \tilde{y})_{u,v}}{{|{R}_{out}(\tilde{y})|}},
\end{equation}
where $\ell_{ce}$ is the pixel-wise cross entropy loss. ${R}_{in}(\tilde{y})$ and ${R}_{out}(\tilde{y})$ denote the inner and outer regions, starting from the contour of $\tilde{y}$ to a distance $d$. Here, $d$ is set as the maximal distance determined by gradually increasing $d$ till either $|{R}_{in}(\tilde{y})|$ or $|{R}_{out}(\tilde{y})|$ no longer changes. In other words, we localize the largest regions that are most likely to contain annotation noise and calculate the corresponding average losses via Eqs. \ref{difficulty1} and \ref{difficulty2}. The two loss values reflect the degree of divergence between $f_{\theta_e}$'s predictions and annotations around contours. It should be noted that using the largest possible regions is to mitigate inaccurate loss values due to wrongly-predicted pixels made by the early model $f_{\theta_e}$, while focusing on symmetric regions (\textit{i.e.}, $|{R}_{in}(\tilde{y})|$ and $|{R}_{out}(\tilde{y})|$) can alleviate the impact of imbalanced sizes between inner and outer regions on divergence estimation (\textit{e.g.}, a larger region will be less divergent as regions distant from contours are more likely to be predicted correctly). 
Thanks to the roughly accurate segmentation results from $f_{\theta_e}$, a larger value of $q_{i,1}$ for client $i$ suggests that it may have larger human-annotated objects, while a larger value of $q_{i,2}$ implies the opposite. Therefore, noise in each client can be estimated by comparing its degree of divergence with that of the whole in globe. To this end, we compute $q_{i,1}$ and $q_{i,2}$ in each client and upload them to the server for obtaining $\mathbf{q}$. Given $\mathbf{q}$, a two-component Gaussian mixture model (GMM) is used to fit $\mathbf{q}$ and classify all clients' indices into two groups: $\mathcal{G}_l$ and $\mathcal{G}_s$, where $\mathcal{G}_l$ and $\mathcal{G}_s$ represent the indices of clients corresponding to Gaussian distributions with mean vectors $(\mu_{l,1},\mu_{l,2})^T$ and $(\mu_{s,1},\mu_{s,2})^T$ respectively under the constraint of $\mu_{l,1}-\mu_{l,2} > \mu_{s,1}-\mu_{s,2}$. In this way, clients in two groups own different patterns of noise. Client $i$ ($i \in \mathcal{G}_l$) tends to annotate objects larger due to greater inconsistency between annotations and model predictions in inner regions, while client $j$ ($j \in \mathcal{G}_s$) tends to annotate objects smaller. Note that this process is privacy-preserving as only average loss values are shared without leaking any information about local data distributions.
After client grouping, noise strength can be estimated in each group by 
\begin{equation} \label{strength}
	s_i = 
	\left\{
	\begin{aligned}
		&q_{i,1} - q_{i,2}, \text{ for } i \text{ in } \mathcal{G}_l,   \\
		&q_{i,2} - q_{i,1}, \text{ for } i \text{ in } \mathcal{G}_s,
	\end{aligned}
	\right.
\end{equation}
\textit{i.e.}, the divergence difference between inner and outer regions. In each group, a greater $s_i$ indicates stronger noise in client $i$.

\subsubsection{Aggregation Weight Calculation}

As discussed above, quantity-based aggregation is vulnerable to noise. To address this, clients with high-quality annotations should be given more attention, which is achievable with the help of noise estimation above. Therefore, we compute a quality-based aggregation weight $\overline{w}_i$ for each client $i$, defined as:
\begin{equation} \label{weight}
	\overline{w}_i = 
	\left\{
	\begin{aligned}
		r \cdot & \frac{\max\limits_{j \in \mathcal{G}_l}s_j - s_i}{|\mathcal{G}_l| \max\limits_{j \in \mathcal{G}_l}s_j - \sum_{t \in \mathcal{G}_l} s_t}, \text{ for } i \text{ in } \mathcal{G}_l,   \\
		(1-r) \cdot & \frac{\max\limits_{j \in \mathcal{G}_s}s_j - s_i}{|\mathcal{G}_s| \max\limits_{j \in \mathcal{G}_s}s_j - \sum_{t \in \mathcal{G}_s} s_t}, \text{ for } i \text{ in } \mathcal{G}_s,
	\end{aligned}
	\right.
\end{equation}
where $r \in [0,1]$ is a balancing coefficient set as 0.5 by default. In this way, a client with higher noise strength is given a smaller weight in each group, denoted as intra-group weighting (IntraGW). However, in real-world scenarios, $\mathcal{G}_l$ and $\mathcal{G}_s$ may differ significantly in size (\textit{i.e.}, $|\mathcal{G}_l|$ and $|\mathcal{G}_s|$), and simply weighting clients within each group can still make the model misled by the majority noise. To address this, we introduce another parameter $r$ to balance the weights between the two groups, denoted as inter-group weighting (InterGW). Our objective is to enable the global model to learn fairly from both annotation noise types, allowing it to learn accurate knowledge that goes beyond noise rather than over-fitting to the majority noise.

Given quality-based weights $\overline{\mathbf{w}}=(\overline{w}_1, \dots, \overline{w}_K)$, together with quantity-based weights $\tilde{\mathbf{w}}=(\tilde{w}_1, \dots, \tilde{w}_K)=\frac{1}{\sum_{i=1}^{K}n_i}(n_1, \dots, n_K)$, one straight-forward way for noise-robust aggregation is to use $\overline{\mathbf{w}}$ to adjust $\tilde{\mathbf{w}}$ for client-wise re-weighting \cite{liang2022rscfed,wicaksana2022fedmix,wu2023fednoro}. However, this method may be sub-optimal as it greatly underestimates the value of low-quality data for model training, especially in shallow layers. As discussed in \cite{bai2021understanding}, noisy labels would affect deep layers more compared to shallow layers. It is because the shallow layers of a model mainly learn general low-level features while the deep layers integrate semantic information for decision-making. In other words, low-quality data can still contribute to the training of shallow layers to maximize the utilization of training data, while high-quality data should dominate the training of deep layers to ensure that the decision-making process is not biased by noise. Therefore, we calculate a noise-robust aggregation weight $w_{i,j}$ for any client $i$ when updating the $j$-th layer of the global model, defined as
\begin{equation} \label{layer}
	w_{i,j} =\mathscr{L}(j)\overline{w}_i + (1-\mathscr{L}(j))\tilde{w}_i, j \in [L],
\end{equation}
where $L$ denotes the total number of layers in the trained model. $\mathscr{L}(j):=\frac{j-1}{L-1}$ is a layer-wise trade-off coefficient which increases linearly as $j$ increases to let $\tilde{w}_i$ and $\overline{w}_i$ dominate the updates of shallow and deep layers respectively. The layer-wise aggregation weight is employed in the subsequent process for noise-robust FL.

\section{Experiments}

\begin{table}[!t]
	\centering
	\renewcommand{\arraystretch}{0.9}
	\resizebox{\columnwidth}{!}{
		\begin{tabular}{c|cc}
			\toprule
			\hline
			\multirow{2}{*}{Noise Type} & \multicolumn{2}{c}{Dataset}              \\ \cline{2-3} 
			& \textit{SKIN}                & \textit{BREAST}             \\ \hline
			$N_S$                        & $M(20, -20, 10, 0.2)$ & $M(25, -15, 5, 0.2)$ \\
			$N_E$                        & $M(25, -25, 10, 0.8)$ & $M(25, -15, 5, 0.8)$ \\ \hline
			\bottomrule
	\end{tabular}}
	\caption{Details of two annotation noise settings.}
	\label{tab:noisesetting}
\end{table}

\begin{table*}[!t]
	\centering
	\renewcommand{\arraystretch}{0.9}
	\begin{tabular}{l|l|l|cccc}
		\toprule
		\hline
		\multirow{4}{*}{Loss}        & \multirow{4}{*}{From} & \multirow{4}{*}{Method} & \multicolumn{4}{c}{Dataset}                                                                                \\ \cline{4-7} 
		&                       &                         & \multicolumn{2}{c|}{\textit{SKIN}}                             & \multicolumn{2}{c}{\textit{BREAST}}       \\ \cline{4-7} 
		&                       &                         & \multicolumn{4}{c}{Noise Type}                                                                         \\ \cline{4-7} 
		&                       &                         & $N_S$  & \multicolumn{1}{c|}{$N_E$}  & $N_S$   & $N_E$   \\ \hline
		\multirow{10}{*}{\textit{CE}} & NeurIPS'18            & GCE                     & $66.43 \pm 2.21$          & \multicolumn{1}{c|}{$73.47 \pm 0.37$}          & $65.67 \pm 1.51$          & $67.12 \pm 0.95$          \\
		& CVPR'19               & SCE                     & $62.49 \pm 4.31$          & \multicolumn{1}{c|}{$74.69 \pm 0.32$}          & $67.71 \pm 1.37$          & $67.73 \pm 0.79$          \\
		& NeurIPS'20            & ELR                     & $65.46 \pm 1.43$          & \multicolumn{1}{c|}{$74.31 \pm 0.41$}          & $65.65 \pm 1.29$          & $67.93 \pm 0.81$          \\
		& CVPR'22               & ADELE                   & $62.99 \pm 3.63$          & \multicolumn{1}{c|}{$70.62 \pm 0.95$}          & $61.79 \pm 1.29$          & $70.15 \pm 2.41$          \\
		& TMI'23                & RMD                     & $62.75 \pm 2.60$          & \multicolumn{1}{c|}{$75.41 \pm 0.37$}          & $64.12 \pm 2.64$          & $68.09 \pm 1.06$          \\
		& MLSys'20              & FedProx                 & $64.35 \pm 1.50$          & \multicolumn{1}{c|}{$75.76 \pm 0.59$}          & $68.10 \pm 1.09$          & $66.55 \pm 1.01$          \\
		& TMI'23                & FedDM                   & $42.44 \pm 0.77$          & \multicolumn{1}{c|}{$67.92 \pm 2.32$}          & $40.76 \pm 2.91$          & $64.16 \pm 3.58$          \\
		& CVPR'22               & FedCorr                 & $61.88 \pm 2.19$         & \multicolumn{1}{c|}{$74.60 \pm 1.17$}        & $67.15 \pm 1.39$       & $68.43 \pm 1.10$                    \\
		& IJCAI'23              & FedNoRo                 & $66.35 \pm 4.62$         & \multicolumn{1}{c|}{$74.22 \pm 0.30$}     & $59.51 \pm 1.28$       & $66.29 \pm 5.03$                    \\
		& Ours                  & FedA$^3$I & $\mathbf{72.69 \pm 1.28}$ & \multicolumn{1}{c|}{$\mathbf{78.81 \pm 0.75}$} & $\mathbf{69.03 \pm 1.95}$ & $\mathbf{74.35 \pm 1.90}$ \\ \hline
		\multirow{3}{*}{\textit{DC}}  & TMI'20                & NR-Dice                 & $73.33 \pm 1.38$          & \multicolumn{1}{c|}{$73.75 \pm 0.68$}          & $74.02 \pm 0.64$          & $65.46 \pm 0.58$          \\
		& TMI'22                & FedMix                  & $67.99 \pm 1.78$          & \multicolumn{1}{c|}{$73.00 \pm 1.22$}          & $72.99 \pm 0.15$          & $69.04 \pm 1.92$          \\
		& Ours                  & FedA$^3$I & $\mathbf{75.92 \pm 0.99}$ & \multicolumn{1}{c|}{$\mathbf{77.05 \pm 0.56}$} & $\mathbf{75.28 \pm 0.53}$ & $\mathbf{76.68 \pm 0.49}$ \\ \hline
		\bottomrule
	\end{tabular}
	\caption{Comparison evaluation measured in the mean (\%) $\pm$ standard deviation (\%) of Dice coefficients against SOTA methods under different annotation noise settings. The best results are marked in bold.}
	\label{tab:sota}
\end{table*}

\subsection{Datasets and Evaluation Metric}

Following prior studies addressing weak/noisy annotations \cite{wicaksana2022fedmix,yao2023learning}, two public medical image segmentation datasets are adopted, including
\begin{enumerate}
	\item The widely-used ISIC 2017 skin lesion segmentation dataset \cite{codella2018skin}, denoted as \textit{SKIN}, contains 2000 images for training and 600 images for testing.
	
	\item Three real-world breast ultrasound datasets for breast tumor segmentation, denoted as \textit{BREAST}, including BUS \cite{al2020dataset}, BUSIS \cite{zhang2022busis}, and UDIAT \cite{yap2017automated}. We perform stratified sampling to split each of the three datasets into a training set and a test set with an 8:2 ratio.
\end{enumerate}
Following standard practice in dealing with noisy annotations \cite{yao2023learning,zhu2019pick,zhang2020robust,li2021superpixel}, all images in \textit{SKIN} and \textit{BREAST} are resized to 256 $\times$ 256 pixels in resolution. The Dice coefficient, as the standard metric for medical image segmentation, is employed for performance evaluation.

\subsection{Experimental Setup}

\subsubsection{Data Partition and Noise Generation}

The training set of each dataset is partitioned into several subsets to simulate FL scenarios. Specifically, we randomly split both \textit{SKIN} and \textit{BREAST} in 50 subsets/clients. For \textit{BREAST}, a client can only contain data from one of BUS, BUSIS, or UDIAT, and thus data across clients is Non-IID due to various sources.

For detailed evaluation, we use the proposed heterogeneous noise model to generate two types of annotation noise, \textit{i.e. } $N_E$ with $p_d > 0.5$ and $N_S$ with $p_d < 0.5$. $N_E$ indicates that the majority of clients tend to annotate objects larger, while $N_S$ represents the opposite. Detailed annotation noise settings are stated in Tab. \ref{tab:noisesetting}.

\subsubsection{Implementation Details}

We select U-Net \cite{ronneberger2015u} as the segmentation model. For FL training, the local epoch and the maximal communication round are set as 5 and 100 respectively. During local training, the model is trained by the Adam optimizer with momentum terms of (0.9, 0.99), a batch size of 8, and a constant learning rate of 5e-3. The warm-up round $T_1$ is set as 10 and 30 by default for \textit{SKIN} and \textit{BREAST} respectively. In each experiment, the training process is independently repeated five times, and the averaged mean and standard deviation of Dice coefficients are reported to eliminate randomness. All experiments are conducted using PyTorch on a NVIDIA RTX-3090 GPU with 24 GB of memory.

\subsection{Comparison with SOTA Methods} \label{sec:compare}

For a comprehensive evaluation of FedA$^3$I on addressing heterogeneous annotation noise, a set of SOTA methods are selected for comparison, including GCE \cite{zhang2018generalized}, SCE \cite{wang2019symmetric}, ELR \cite{liu2020early}, ADELE \cite{liu2022adaptive}, RMD \cite{fang2023reliable}, NR-Dice \cite{wang2020noise}, FedProx \cite{FedProx}, FedMix \cite{wicaksana2022fedmix}, FedDM \cite{zhu2023feddm}, FedCorr \cite{xu2022fedcorr}, and FedNoRo \cite{wu2023fednoro}. The introduction and implementation details of these methods can be found in Appendix B.

The above methods can be divided into two groups according to the used loss functions: Cross entropy loss-based methods (denoted as \textit{CE}) and Dice loss-involved methods (denoted as \textit{DC}). For a fair comparison, FedA$^3$I is implemented by two types of loss functions separately, namely 1) cross-entropy loss and 2) a combination of cross-entropy loss and Dice loss \cite{milletari2016v}. Quantitative comparison results on both datasets under two noise settings are summarized in Tab. \ref{tab:sota}. Compared to other methods, FedA$^3$I achieves the best performance across different datasets and settings, validating its effectiveness in handling heterogeneous annotation noise. It confirms the necessity of specific designs for the two-level Non-IID annotation noise in FMIS. In addition, different from those methods introducing complex correction \cite{xu2022fedcorr,liu2022adaptive,zhu2023feddm} or regularization \cite{liu2020early,wu2023fednoro}, no additional training overhead is introduced in FedA$^3$I, making it more easy-to-deploy in practice.

\begin{table*}[!t]
	\centering
	\renewcommand{\arraystretch}{0.9}
	\begin{tabular}{l|ccc|cccc}
		\toprule
		\hline
		\multicolumn{1}{l|}{\multirow{4}{*}{Loss}} & \multicolumn{3}{c|}{\multirow{3}{*}{Component}}                                   & \multicolumn{4}{c}{Dataset}                                               \\ \cline{5-8} 
		\multicolumn{1}{l|}{}                       & \multicolumn{3}{c|}{}                                                             & \multicolumn{2}{c|}{\textit{SKIN}}                    & \multicolumn{2}{c}{\textit{BREAST}} \\ \cline{5-8} 
		\multicolumn{1}{l|}{}                       & \multicolumn{3}{c|}{}                                                             & \multicolumn{4}{c}{Noise Type}                                            \\ \cline{2-8} 
		\multicolumn{1}{l|}{}                       & FedAvg                    & IntraGW                   & InterGW                   & $N_S$      & \multicolumn{1}{c|}{$N_E$}      & $N_S$        & $N_E$       \\ \hline
		\multirow{3}{*}{\textit{CE}}                         & \checkmark &                           &                           & $63.85 \pm 2.41$ & \multicolumn{1}{c|}{$75.69 \pm 0.56$} & $65.33 \pm 1.03$   & $68.49 \pm 0.67$  \\
		& \checkmark & \checkmark &                           & $67.45 \pm 1.55$ & \multicolumn{1}{c|}{$77.33 \pm 0.18$} & $67.87 \pm 0.87$   & $69.26 \pm 1.80$  \\
		& \checkmark & \checkmark & \checkmark & $\mathbf{72.69 \pm 1.28}$ & \multicolumn{1}{c|}{$\mathbf{78.81 \pm 0.75}$} & $\mathbf{69.03 \pm 1.95}$   & $\mathbf{74.35 \pm 1.90}$  \\ \hline
		\multirow{3}{*}{\textit{DC}}                         & \checkmark &                           &                           & $73.50 \pm 1.24$ & \multicolumn{1}{c|}{$75.06 \pm 0.41$} & $72.31 \pm 1.11$   & $67.27 \pm 0.70$  \\
		& \checkmark & \checkmark &                           & $74.59 \pm 1.03$ & \multicolumn{1}{c|}{$76.70 \pm 0.52$} & $74.49 \pm 0.94$   & $68.31 \pm 0.46$  \\
		& \checkmark & \checkmark & \checkmark & $\mathbf{75.92 \pm 0.99}$ & \multicolumn{1}{c|}{$\mathbf{77.05 \pm 0.56}$} & $\mathbf{75.28 \pm 0.53}$   & $\mathbf{76.68 \pm 0.49}$  \\ \hline
		\bottomrule
	\end{tabular}
	\caption{Component-wise evaluation measured in the mean (\%) $\pm$ standard deviation (\%) of Dice coefficients under different annotation noise settings. The best results are marked in bold.}
	\label{tab:ablation1}
\end{table*}

\subsection{Ablation Study}

To analyze the effectiveness of FedA$^3$I, a series of in-depth ablation studies are conducted.

\begin{figure}[!t] 
	\centering
	\includegraphics[width=\columnwidth]{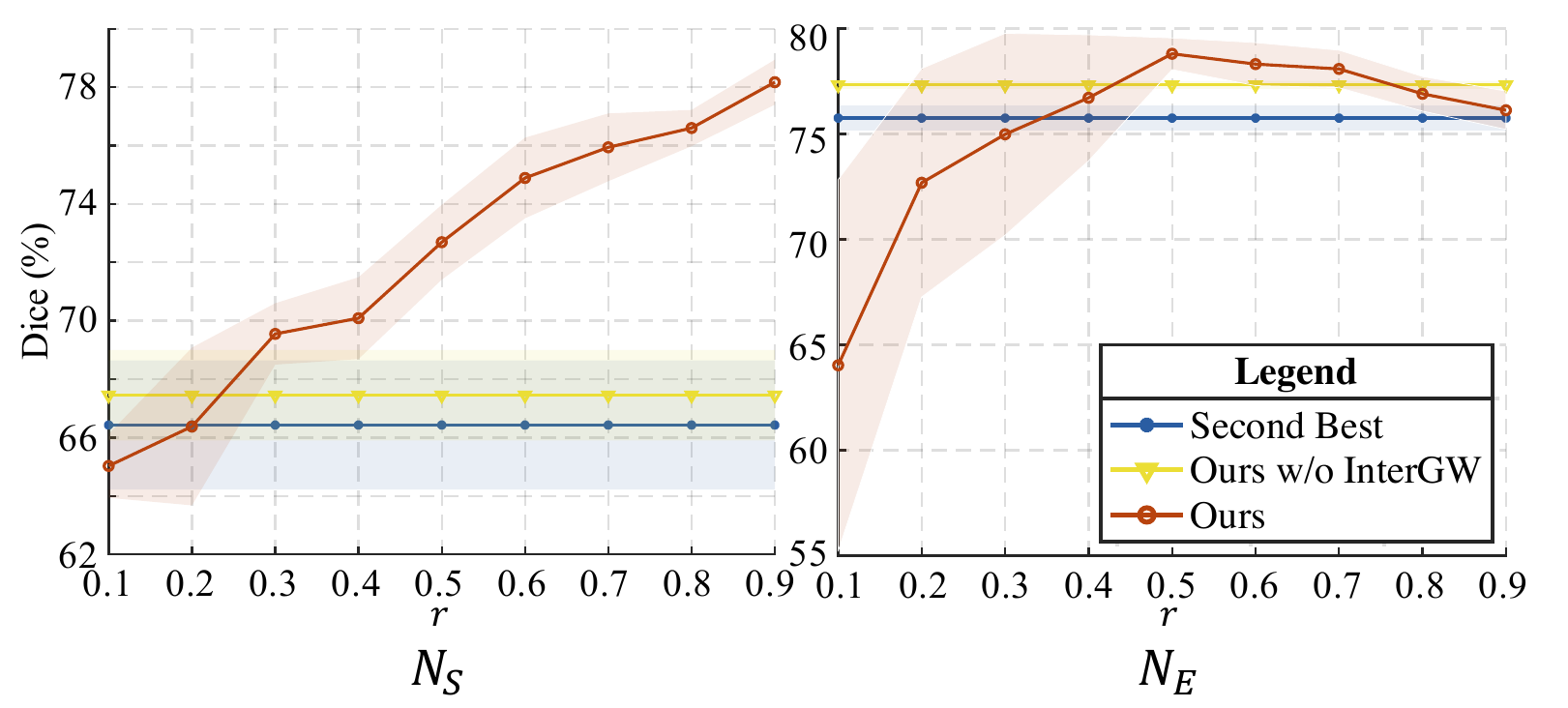}
	\caption{Ablation studies on the balance coefficient $r$ in InterGW. The solid line and transparent areas represent the mean and standard deviation respectively. \textit{Second best} means the performance of the second-best method in Tab. \ref{tab:sota}.}
	\label{fig:r}
\end{figure}

\subsubsection{Component-wise Ablation Study}

IntraGW and InterGW are added to the quantity-based aggregation strategy (FedAvg) sequentially to validate whether quality-based factors can help aggregation with underlying noisy clients/models. Experiments are conducted following the settings in Sec. \ref{sec:compare}, and quantitative results are summarized in Tab. \ref{tab:ablation1}. FedAvg struggles with noisy annotations, as it measures the contribution of each client only based on its data amount, making the global model biased towards clients with large amounts of noisy data. To mitigate this, IntraGW is introduced to up-weight the clients with higher-quality annotations, leading to a performance boost. Further introducing InterGW to balance the two client groups $\mathcal{G}_l$ and $\mathcal{G}_s$ would help the model learn general knowledge instead of over-fitting to the majority noise, achieving the best performance.

\subsubsection{Impact of the Balance Coefficient $r$ in InterGW}

In InterGW, $r$, set as 0.5 by default, is to balance the two client groups $\mathcal{G}_l$ and $\mathcal{G}_s$. To evaluate the impact of $r$, ablation studies under various settings of $r$ are conducted on \textit{SKIN} with noise settings stated in Tab \ref{tab:noisesetting}, and quantitative results are illustrated in Fig. \ref{fig:r}. Comparison results indicate that $r$ is a noise-related parameter, and InterGW can even fail without a proper setting. It is because $r$ is to re-weight $\mathcal{G}_l$ and $\mathcal{G}_s$, and the optimal weighting strategy can vary under different noise settings (\textit{e.g.} $N_S$ and $N_E$). Fortunately, using the default value of 0.5 brings a stable boost for InterGW, leading to better performance than the second-best SOTA method.

\begin{figure}[!t] 
	\centering
	\includegraphics[width=\columnwidth]{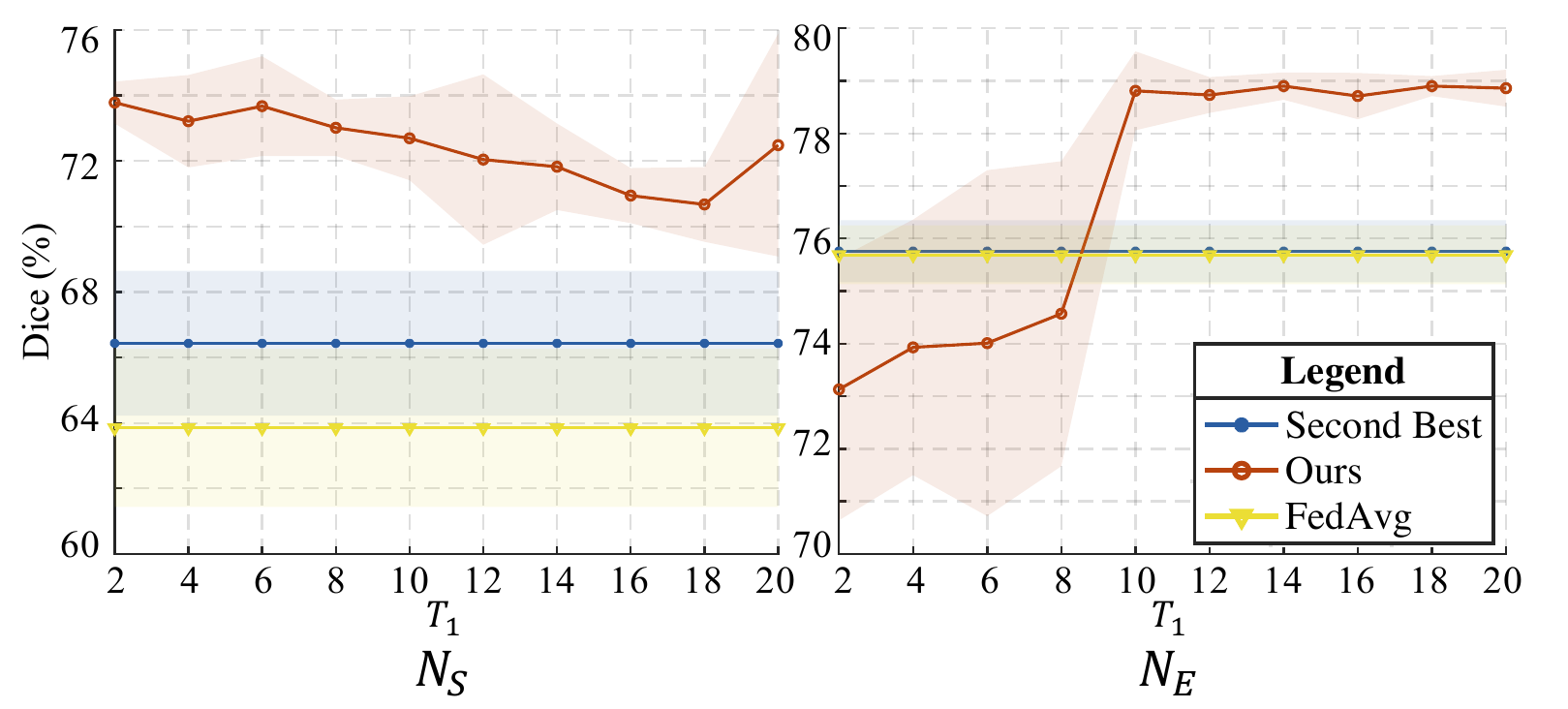}
	\caption{Ablation studies on the warm-up round $T_1$. The solid line and transparent areas represent the mean and standard deviation respectively. \textit{Second best} means the performance of the second-best method in Tab. \ref{tab:sota}.}
	\label{fig:rounds}
\end{figure}

\subsubsection{Impact of the Warm-up Round $T_1$}

In FedA$^3$I, we perform noise estimation with a warm-up global model trained by FedAvg. A natural question arises: how many training rounds are needed for warm-up training? Therefore, ablation studies under different settings of $T_1$ are conducted on \textit{SKIN} as illustrated in Fig \ref{fig:rounds}. FedA$^3$I stably outperforms FedAvg and other SOTA methods with $T_1$ in a certain range (\textit{i.e.} 10-20), demonstrating that FedA$^3$I is relatively robust to the selection of $T_1$. Hence, during the early stage of model learning (but not too early), $T_1$ can be chosen relatively arbitrarily to reduce the workload of hyper-parameter tuning.

\section{Conclusion}
This paper presents a novel FMIS problem: training a segmentation model with decentralized data containing heterogeneous cross-client annotation noise. To formulate this problem, we first propose a general model called CEM to generate annotator-specific noise for pixel-wise annotations, and then extend it to construct a heterogeneous cross-client noise model for multi-source data. For problem-solving, we propose FedA$^3$I for noise-robust FMIS. Specifically, a quality factor, calculated based on noise estimation, is incorporated into the model aggregation phase in a layer-wise manner to emphasize the contribution of high-quality clients. Experimental results on two publicly-available datasets demonstrate the superiority of FedA$^3$I against the state-of-the-art approaches. We believe that the presented problem, the noise model, and the solution FedA$^3$I will serve as valuable inspiration for future research in developing realistic medical image segmentation systems from decentralized data.

\newpage
\section*{Acknowledgments}
This work was supported in part by the National Natural Science Foundation of China under Grants 62202179 and 62271220 and in part by the Natural Science Foundation of Hubei Province of China under Grant 2022CFB585. The computation is supported by the HPC Platform of HUST. We also thank Dr. Zhenyu Liao for his assistance in mathematical proof.

\bibliography{aaai24}

\twocolumn
\newpage
\appendix
\onecolumn
\part*{\centering Appendix}

This appendix is organized as follows.
\begin{itemize}
	
	\item In Section \ref{sec:noise}, we discuss annotation noise further.
	
	\item In Section \ref{sec:experiment}, experimental details are presented.
	
\end{itemize}
Please note that the references cited in the appendix can be found in the main text. In addition, the tables and figures are also sequentially numbered along with the corresponding text to facilitate easy referencing.

\section{Noise} \label{sec:noise}
\subsection{Discussion about Pixel-Dependent Property}
We utilize Fig \ref{fig:example} as a toy example to illustrate the underlying motivation behind pixel-dependent noise (PDN) as defined in Definition \ref{PDN}. Note that segmenting of this circular object is achieved by delineating its contour, and noise exists in the pixels between the marked and the clean (\textit{i.e.}, ideal) contours. This observation indicates that noise is distributed in the vicinity of the contour, which corresponds to the first condition outlined in Definition \ref{PDN}. Additionally, Fig. \ref{fig:example} illustrates the challenge of delineating the upper half of the object due to the non-distinct boundary caused by continuous transitions, while the lower half is comparatively easier to delineate owing to its distinct boundary. This phenomenon leads to Non-IID noise along the tangential direction (\textit{i.e.}, noise is stronger near the boundary in the upper half), which is frequently observed in medical images due to the anisotropic nature of medical objects. The second condition in Definition \ref{PDN} captures this attribute.

However, existing noise models fail to comprehensively capture this Non-IID across-pixel noise. When extending label noise for the classification task like class-conditional noise (CCN) \cite{han2018co,yu2019does,chen2021beyond} to model annotation noise, the spatial information is disregarded since each pixel is treated as an independent classification instance. Consequently, neither condition is satisfied in this scenario. While morphological operation noise (MON) \cite{zhang2020disentangling,liu2022adaptive} and Markov label noise (MLN) \cite{yao2023learning} fit the first condition, they overlook the second condition due to the assumption of an identical noise distribution along the tangent direction (\textit{i.e.} the identical Bernoulli distribution).

Different from previous noise models, our presented contour evolution model (CEM) offers a stronger guarantee of pixel-dependent noise. Firstly, we introduce noise through a biased contour. As a result, the noise is inherently distributed around the boundary, satisfying the first condition. Secondly, the distribution of each element in the bias sequence $\langle b_k \rangle_{k=1}^l$ is not strictly identically distributed as Theorem \ref{th:non_IID_b}, resulting in Non-IID distortion along the tangential direction. This satisfies the second condition.

\begin{theorem} \label{th:non_IID_b}
	Assume there are $l_{sub}$ pairs $(u_1, v_1), (u_2, v_2), \dots, (u_{l_{sub}}, v_{l_{sub}})$, where $u_1, u_2, \dots, u_{l_{sub}}$ are differently selected from $\{1, 2, \dots, l\} (l_{sub} < l)$ with a determined strategy, and $v_1, v_2, \dots, v_{l_{sub}}$ \resizebox{!}{2ex}{$\stackrel{\text{i.i.d.}}{\sim}$} $\bm{N}(\mu,\sigma^2)$. These $l_{sub}$ pairs are used to fit a $p$-th ($p \leq l_{sub}-1$) degree polynomial function $\mathscr{P}(u)=a_pu^p+a_{p-1}u^{p-1}+\dots+a_1u+a_0$ by minimize the mean square error. We have $\mathscr{P}(1), \mathscr{P}(2), \dots, \mathscr{P}(l)$ are not strictly identically distributed.
\end{theorem}

\begin{proof}[\textbf{Proof}]\renewcommand{\qedsymbol}{}
	The basis $\{1, u, \dots, u^p\}$ satisfy the \textbf{Haar condition} in the set $\{u_1, u_2, \dots, u_{l_{sub}}\}$. Hence, according to interpolation theory, the function $\mathscr{P}(\cdot)$ is unique and satisfying:
	$$\frac{\partial \sum_{i=1}^{l_{sub}} (\sum_{j=0}^{p}a_j u_i^j -v_i)^2}{\partial a_k} = 0, k \in \{0, 1, \dots, p\}.$$
	We further have:
	$$\sum_{j=0}^{p}a_j \sum_{i=1}^{l_{sub}} u_i^{k+j} = \sum_{i=1}^{l_{sub}} u_i^k v_i, k \in \{0, 1, \dots, p\}.$$
	In matrix form, we have:
	$$\left[\begin{array}{cccc}
		l_{sub} & \sum_{i=1}^{l_{sub}} u_{i} & \cdots & \sum_{i=1}^{l_{sub}} u_{i}^{p} \\
		\sum_{i=1}^{l_{sub}} u_{i} & \sum_{i=1}^{l_{sub}} u_{i}^{2} & \cdots & \sum_{i=1}^{l_{sub}} u_{i}^{p+1} \\
		\vdots & \vdots & & \vdots \\
		\sum_{i=1}^{l_{sub}} u_{i}^{p} & \sum_{i=1}^{l_{sub}} u_{i}^{p+1} & \cdots & \sum_{i=1}^{l_{sub}} u_{i}^{2p}
	\end{array}\right]\left[\begin{array}{c}
		a_{0} \\
		a_{1} \\
		\vdots \\
		a_{p}
	\end{array}\right]=\left[\begin{array}{c}
		\sum_{i=1}^{l_{sub}} v_{i} \\
		\sum_{i=1}^{l_{sub}} u_{i} v_{i} \\
		\vdots \\
		\sum_{i=1}^{l_{sub}} u_{i}^{p} v_{i}
	\end{array}\right]$$
	The matrix is determined because the selection strategy is determined. In addition, \textbf{Haar condition} ensure it is non-singular. Based on this, we have:
	$$\left[\begin{array}{c}
		a_{0} \\
		a_{1} \\
		\vdots \\
		a_{p}
	\end{array}\right]=
	\left[\begin{array}{cccc}
		l_{sub} & \sum_{i=1}^{l_{sub}} u_{i} & \cdots & \sum_{i=1}^{l_{sub}} u_{i}^{p} \\
		\sum_{i=1}^{l_{sub}} u_{i} & \sum_{i=1}^{l_{sub}} u_{i}^{2} & \cdots & \sum_{i=1}^{l_{sub}} u_{i}^{p+1} \\
		\vdots & \vdots & & \vdots \\
		\sum_{i=1}^{l_{sub}} u_{i}^{p} & \sum_{i=1}^{l_{sub}} u_{i}^{p+1} & \cdots & \sum_{i=1}^{l_{sub}} u_{i}^{2p}
	\end{array}\right]^{-1}\left[\begin{array}{c}
		\sum_{i=1}^{l_{sub}} v_{i} \\
		\sum_{i=1}^{l_{sub}} u_{i} v_{i} \\
		\vdots \\
		\sum_{i=1}^{l_{sub}} u_{i}^{p} v_{i}
	\end{array}\right]$$
	To this end, $\{a_0, \dots, a_p\}$ can be seen as different linear combinations of $\{v_1, \dots, v_{l_{sub}}\}$. Let's assume:
	\begin{align*}
		a_0 = \alpha_{01} v_1 + \alpha_{02} v_2 + \dots + \alpha_{0l_{sub}} v_{l_{sub}},\\
		a_1 = \alpha_{11} v_1 + \alpha_{12} v_2 + \dots + \alpha_{1l_{sub}} v_{l_{sub}},\\
		\vdots \quad \quad \quad \quad \quad \quad \quad \quad \\
		a_p = \alpha_{p1} v_1 + \alpha_{p2} v_2 + \dots + \alpha_{pl_{sub}} v_{l_{sub}}.
	\end{align*}
	We further denote:
	\begin{align*}
		\mathscr{P}(1) = \sum_{j=0}^{p} a_j 1^j = \beta_{11} v_1 + \beta_{12} v_2 + \dots + \beta_{1l_{sub}} v_{l_{sub}},\\
		\mathscr{P}(2) = \sum_{j=0}^{p} a_j 2^j = \beta_{21} v_1 + \beta_{22} v_2 + \dots + \beta_{2l_{sub}} v_{l_{sub}},\\
		\vdots \quad \quad \quad \quad \quad \quad \quad \quad \quad \quad \quad \\
		\mathscr{P}(l) = \sum_{j=0}^{p} a_j l^j = \beta_{l1} v_1 + \beta_{l2} v_2 + \dots + \beta_{ll_{sub}} v_{l_{sub}}.\\
	\end{align*}
	We have,
	$$\mathscr{P}(j) \sim \bm{N}(\sum_{i=1}^{l_{sub}}\beta_{ji} \mu, \sum_{i=1}^{l_{sub}}\beta_{ji} \sigma^2), j \in \{1, 2, \dots, l\}$$
	Therefore, $\mathscr{P}(1), \mathscr{P}(2), \dots, \mathscr{P}(l)$ are not strictly identically distributed.
\end{proof}

\begin{figure}[!t] 
	\centering
	\includegraphics[width=0.3\columnwidth]{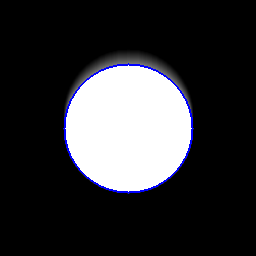}
	\caption{A circular object and its boundary (\textit{i.e.}, the blue curve).}
	\label{fig:example}
\end{figure}

\subsection{Visualization of CEM}
Noise generated by six CEMs are visualized as Fig. \ref{fig:v_cem} to show the effectiveness. The two parameters control the bias of the generated contour, leading to different noise strength and patterns in annotations.

\begin{figure}[!t] 
	\centering
	\subfigure[$C(-5, 1)$]{\includegraphics[width=0.31\columnwidth]{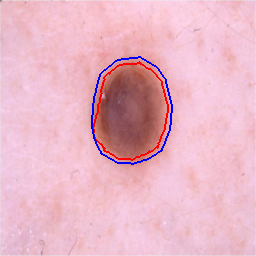}}
	\subfigure[$C(-7, 2)$]{\includegraphics[width=0.31\columnwidth]{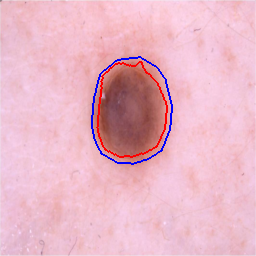}}
	\subfigure[$C(-10, 2)$]{\includegraphics[width=0.31\columnwidth]{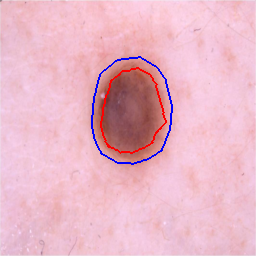}}
	\subfigure[$C(5, 2.5)$]{\includegraphics[width=0.31\columnwidth]{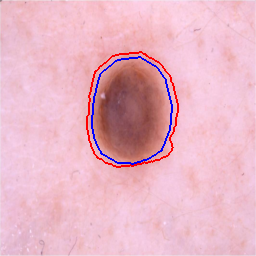}}
	\subfigure[$C(10, 4)$]{\includegraphics[width=0.31\columnwidth]{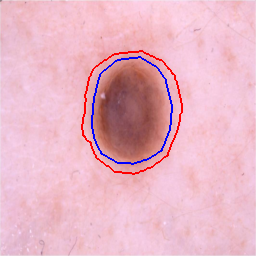}}
	\subfigure[$C(15, 5)$]{\includegraphics[width=0.31\columnwidth]{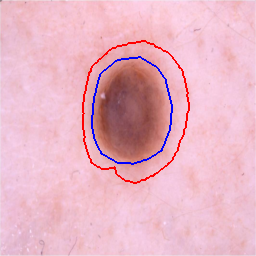}}
	\caption{Visualization of noise generated by different CEMs. The blue and red curves represent the contours of clean (\textit{i.e.}, ideal) and noisy annotations, respectively.}
	\label{fig:v_cem}
\end{figure}

\section{Experiments} \label{sec:experiment}

\subsection{Details of SOTA Methods for Comparison}
Introduction and implementation details of methods for comparison are given as follows.

\begin{itemize}
	
	\item GCE \cite{zhang2018generalized}: Generalized cross entropy (GCE) loss can be seen as a trade-off between cross entropy loss (CE) and mean absolute error loss. In all experiments, the hyper-parameter $q$ in GCE is set as 0.7 following the original paper.
	
	\item SCE \cite{wang2019symmetric}: Symmetric cross entropy (SCE) loss combines the reverse CE loss with the vanilla CE loss for robust learning with noisy labels. We set $\alpha=0.5$ and $\beta=1$ in SCE following \cite{yao2023learning} in experiments.
	
	\item ELR \cite{liu2020early}: Early learning regularization (ELR) is a technique that can prevent models from fitting to noisy labels during the training process. Following the paper, we set hyper-parameters $\beta=0.99$ and $\lambda=1$.
	
	\item ADELE \cite{liu2022adaptive}: Adaptive early-learning correction can handle IID annotation noise, but it may experience a decrease in performance when the noise is heterogeneous. In the context of FL, the IoU metric is computed on the global training dataset for label correction. Following the original paper, we set the hyper-parameters $\lambda$, $\rho$, $r$, and $\tau$ to 1, 0.8, 0.9, and 0.8, respectively, for both datasets in experiments. 
	
	\item RMD \cite{fang2023reliable}: Reliable mutual distillation (RMD) is a solution to address annotation noise for medical images. $t_1$, $\mu$ and $\tau$ are set as 10, 0.8 and 0.005, respectively, following the paper.
	
	\item NR-Dice \cite{wang2020noise}: NR-Dice loss is a special Dice loss to learn from noisy pixel-wise annotations. We set the hyper-parameter $\gamma$ as 1.5 according to the original paper.
	
	\item FedProx \cite{FedProx}: The local proximal regularization term is proven to be robust to label noise in \cite{xu2022fedcorr}. Therefore, FedProx is selected as a baseline for comparison. The weight of the proximal term $\mu_{Prox}$ is tuned from $\{0.01, 0.001\}$.
	
	\item FedMix \cite{wicaksana2022fedmix}: FedMix focuses on the mixed-type annotations in FL. In the experiments, we treat noisy annotations as pixel-wise annotations instead of weak annotations when implementing FedMix. This is necessary because using weak annotations leads to training crashes. For hyper-parameters, following the orgion paper, we set $\lambda$ as 10, and $\beta$ as 2.5 and 1.5 for \textit{SKIN} and \textit{BREAST}, respectively.
	
	\item FedDM \cite{zhu2023feddm}: FedDM addresses typical bounding box annotations in FL. However, its strong assumption that the bounding box must completely contain the target area can make it invalid when dealing with generic annotation noise. We implement FedDM with CE loss for fair comparison.
	
	\item FedCorr \cite{xu2022fedcorr}: FedCorr is a multi-stage FL framework for the classification task with noisy labels. Due to the high cost of calculating LID score for the segmentation task, we replace LID score with loss value to identify noisy clients. $T_1$, $T_2$ and $T_3$ are set as 2, 8 and 90, respectively.
	
	\item FedNoRo \cite{wu2023fednoro}: FedNoRo is a two-stage FL framework for class-imbalanced noisy data. For fair comparison, we use vanilla CE loss for local training. The warm-up round is set following the paper.
	
\end{itemize}
To ensure a fair comparison, the same settings are used for each dataset that are not explicitly mentioned above. These shared settings include the backbone model, the communication round, the local epoch, the optimizer, \textit{etc}.

\end{document}